\documentclass{article}
\usepackage{kr}
\usepackage{times}
\usepackage{url}
\usepackage[utf8]{inputenc}
\usepackage[small]{caption}
\usepackage{graphicx}
\usepackage{amsmath}
\usepackage{amsthm}
\usepackage{booktabs}
\usepackage{amssymb}
\usepackage{multirow}
\usepackage{mathrsfs}
\usepackage{stmaryrd}
\usepackage{listings}
\usepackage[usenames,dvipsnames,svgnames,table]{xcolor}
\usepackage[unicode,pdfmenubar,linktoc=all,hidelinks,bookmarks, pdftex]{hyperref}
\usepackage{xspace}
\usepackage{graphicx}
\usepackage{csquotes}
\usepackage{xparse}
\usepackage{booktabs}
\usepackage{enumerate}
\usepackage[kerning,spacing]{microtype}
\usepackage{comment}
\usepackage{tikzlings}

\usepackage{dlcobook}
\renewcommand{\pdl}{\ensuremath{\PL(\dep{,})}\xspace} 

\newif\ifhideproofs

\ifhideproofs
\usepackage{environ}
\NewEnviron{hide}{}

\fi

\newtheorem{theorem}{Theorem}
\newtheorem{proposition}[theorem]{Proposition}
\newtheorem{lemma}[theorem]{Lemma}
\newtheorem{corollary}[theorem]{Corollary}
\newtheorem{fact}[theorem]{Fact}

\theoremstyle{definition}
\newtheorem{definition}[theorem]{Definition}
\newtheorem{example}[theorem]{Example}

\newcommand{\minOf}[2]{\ensuremath{\min(#1,#2)}}
\renewcommand{\tuple}[1]{\ensuremath{\langle{#1}\rangle}}
\newcommand{\statesOf}[1]{\ensuremath{S(#1)}}
\newcommand{\modelsOf}[1]{\ensuremath{\llbracket #1\rrbracket}}

\newcommand{\downset}[1]{\ensuremath{\mathrm{down}(#1)}}

\newcommand{\Int}{\ensuremath{X}}   %

\newcommand{\pmW}{\ensuremath{\mathbb{W}}}
\newcommand{\pmS}{\ensuremath{\mathcal{S}}}
\newcommand{\pmL}{\ensuremath{\ell}}
\newcommand{\pmP}{\ensuremath{\prec}}

\newcommand{\pmWpeng}{\ensuremath{\pmW_{\!\text{peng}}}}
\newcommand{\pmSpeng}{\ensuremath{\pmS_{\text{peng}}}}
\newcommand{\pmLpeng}{\ensuremath{\pmL_{\text{peng}}}}
\newcommand{\pmPpeng}{\ensuremath{\pmP_{\text{peng}}}}

\newcommand{\pmWsub}{\ensuremath{\pmW_{\!\text{sub}}}}
\newcommand{\pmSsub}{\ensuremath{\pmS_{\text{sub}}}}
\newcommand{\pmLsub}{\ensuremath{\pmL_{\text{sub}}}}
\newcommand{\pmPsub}{\ensuremath{\pmP_{\text{sub}}}}

\newcommand{\pmWsup}{\ensuremath{\pmW_{\!\text{sup}}}}
\newcommand{\pmSsup}{\ensuremath{\pmS_{\text{sup}}}}
\newcommand{\pmLsup}{\ensuremath{\pmL_{\text{sup}}}}
\newcommand{\pmPsup}{\ensuremath{\pmP_{\text{sup}}}}

\newcommand{\pmWpq}{\ensuremath{\pmW_{\!\text{pq}}}}
\newcommand{\pmSpq}{\ensuremath{\pmS_{\text{pq}}}}
\newcommand{\pmLpq}{\ensuremath{\pmL_{\text{pq}}}}
\newcommand{\pmPpq}{\ensuremath{\pmP_{\text{pq}}}}

\newcommand{\nmableitW}{\nmableit_{\!\pmW}}
\newcommand{\nmableitWparam}[1]{\nmableit_{\!#1}}
\newcommand{\notnmableitW}{\notnmableit_{\!\pmW}}
\newcommand{\notnmableitWparam}[1]{\notnmableit_{\!#1}}

\ifx\nmableit\undefined
\makeatletter
\ifx\centernot\undefined
\newcommand*{\centernot}{%
	\mathpalette\@centernot
}
\fi
\def\@centernot#1#2{%
	\mathrel{%
		\rlap{%
			\settowidth\dimen@{$\m@th#1{#2}$}%
			\kern.5\dimen@
			\settowidth\dimen@{$\m@th#1=$}%
			\kern-.5\dimen@
			$\m@th#1\not$%
		}%
		{#2}%
	}%
}
\makeatother
\DeclareRobustCommand\nmableitSymb{\mathrel|\mkern-.5mu\joinrel\sim} %
\newcommand{\nmableit}{\ensuremath{\mbox{$\,\nmableitSymb\,$}}} %
\newcommand{\notnmableit}{\ensuremath{\mbox{$\,\centernot\nmableitSymb\,$}}} %
\fi
\title{A Primer for Preferential Non-Monotonic Propositional Team Logics}
\author{%
    Kai Sauerwald$^1$ \and Juha Kontinen$^2$
    \affiliations
    $^1$Artifical Intelligence Group, FernUniversität in Hagen, Hagen, Germany\\
    $^2$Department of Mathematics and Statistics, University of Helsinki, Helsinki, Finland
}

\begin{document}

\maketitle
\pagestyle{plain}

\begin{abstract}
This paper considers KLM-style preferential non-monotonic reasoning in the setting of propositional team semantics.
We show that team-based propositional logics naturally give rise to cumulative non-monotonic entailment relations. Motivated by the non-classical interpretation of disjunction in team semantics,  we give a precise characterization for preferential models for propositional dependence logic satisfying all of System P postulates.  Furthermore, we show how classical entailment and dependence logic entailment can be expressed in terms of non-trivial preferential models.
\end{abstract}

\section{Introduction}
We define non-monotonic versions of team-based logics and study their axiomatics regarding System~P. The logics are defined with the aid of preferential models in the style of \citeauthor{KS_KrausLehmannMagidor1990} (KLM, \citeyear{KS_KrausLehmannMagidor1990}).

Team semantics is a logical framework for studying concepts and phenomena that arise in the presence of plurality of data. Prime examples of such concepts are, e.g.,   functional dependence ubiquitous in database theory and conditional independence of random variables in statistics. %
 The beginning of the field of team semantics can be traced back to the introduction of (first-order) dependence logic in \cite{vaananen07}.
In dependence logic, formulas are interpreted by sets of assignments (teams). Syntactically, dependence logic extends first-order logic by dependence atoms $\dep{\vec{x},y}$ expressing that the values of the variables $\vec x$ functionally determine the value of the variable $y$. %
Inclusion logic \cite{galliani12} is another prominent logic in this context that extends first-order logic by inclusion atoms $\vec{x} \subseteq \vec{y}$, whose interpretation corresponds exactly to that of inclusion dependencies in database theory. 
During the past decade, the expressivity and complexity aspects of logics in team semantics have been extensively studied.
Fascinating connections have been drawn to areas such as database theory \cite{HannulaKV20,HannulaK16}, verification \cite{GutsfeldMOV22}, real-valued computation  \cite{abs-2003-00644}, inquisitive logic \cite{10.1215/00294527-2019-0033}, and epistemic logic \cite{Galliani15}. These works have focused on logics in the first-order, propositional and modal team semantics, and more recently also in the multiset \cite{DurandHKMV18}, probabilistic \cite{HKMV18} and semiring settings \cite{BarlagHKPV23}. As far as the authors know, a merger of logics in team semantics and non-monotonic reasoning has not been studied so far.

Non-monotonicity is one of the core phenomenons of reasoning that are deeply studied in knowledge representation and reasoning; see Gabbay et al. (\citeyear{KS_GabbayHoggerRobinson1993}) and Brewka et al. (\citeyear{KS_Brewka1997}) for an overview, with, e.g., connections to belief change \cite{KS_MakinsonGaerdenfors1991} and human-like reasoning \cite{KS_RagniKernIsbernerBeierleSauerwald2020}.
Non-monotonic inference \( \varphi \nmableit \psi \) 
is often understood
as \enquote{when \( \varphi \) holds, then usually \( \psi \) holds}, where usually can be understood in the sense of \emph{expected} \cite{KS_GaerdenforsMakinson1994}. 
One can imagine adapting this notion of non-monotonic inference to propositional team logics.
For instance, in dependence logic, an entailment \( \dep{b,f} \models \neg p \) states that 
    \emph{\enquote{when whether it is a \emph{bird} (\( b \)) determines whether it \emph{flies} (\( f \)), then it is not a penguin (\( \neg p \))}}
and an analogue non-monotonic entailment \( \dep{b,f} \nmableit \neg p \) can be read as 
    \emph{\enquote{when whether it is a \emph{bird} (\( b \)) determines whether it \emph{flies} (\( f \)), then it is \textbf{usually} not a penguin (\( \neg p \))}}.
For the latter kind of expression, there is no obvious way to formulate it in existing team-based logic, so injecting non-monotonicity is a valuable extension of team logics.
For a start, one can rely on the basic systems of non-monotonic reasoning.
The very most basic denominator of non-monotonic reasoning is often denoted cumulative reasoning, which is given axiomatically by \emph{System~C} \cite{KS_Gabbay1984}.
In extension to cumulative reasoning, non-monotonic reasoning in the style of KLM
 is considered as the \enquote{conservative core of non-monotonic reasoning} \cite{KS_Pearl1989,KS_Gabbay1984}.
KLM-style non-monotonic reasoning has two prominent representations (KLM \citeyear{KS_KrausLehmannMagidor1990}): 
\begin{enumerate}[(KLM.1)]
	\item reasoning over \emph{preferential models}; and 
	\item an axiomatic characterization, called \emph{System P}, which is an extension of System C. 
\end{enumerate}
Because of (KLM.1), KLM-style reasoning is also denoted \emph{preferential reasoning}.
Common for both representations of KLM-style reasoning is, that they are parametric in the sense that they make use of some underlying classical logic \( \mathscr{L} \), e.g., propositional logic or first-order logic. 
In this paper, we define preferential team logics via preferential models (as in KLM.1).
The rationale is that we think that preferential models capture the original intention of preferential logic best, and, as we demonstrate, it shows standard non-monotonic behaviour.
Furthermore, we study the relationship of preferential teams logic to System~P (as in KLM.2).
Our axiomatic studies show that for general team-based logics, (KLM.1) and (KLM.2) do not induce the same non-monotonic inference relations. 
This is of interest, e.g., because it gives a negative answer to the question of whether the relationship between (KLM.1) and (KLM.2) by KLM (\citeyear{KS_KrausLehmannMagidor1990}) generalize beyond the assumptions by KLM\footnote{KLM assume a compact Tarskian logic with Boolean connectives. In team logics  (by default), there is no negation, and disjunction is non-classical, i.e., it does not behave like Boolean disjunction.}.
We give a condition for preferential models that is sufficient to reestablish satisfaction of System~P in all preferential team logics. 
Specifically for preferential dependence logic, we also show that this condition exactly characterizes those preferential models such that System~P is satisfied.
Moreover, when using specific (non-trivial) preferences, preferential dependence logic becomes dependence logic, respectively, it is equivalent to classical propositional entailment.

We have proof for all statements in this paper; they are available in the accompanying supplementary material.

\section{Propositional Team-Based Logics}\label{sec:teambasedlogics}
\noindent\textbf{Propositional Logic with Team Semantics.}
We denote by $\Prop=\{p_i\ : \ i\in \mathbb{N}\}$ the set of propositional variables. 
We will use letters $p,q,r,\dots$ (with or without subscripts) to stand for elements of $\Prop$.
\begin{definition}[Classical propositional logic (\PL)]
Well formed \PL-formulas $\alpha$  are formed by the grammar:
\[\alpha::=p\mid \neg p \mid \bot\mid\top\mid \alpha\wedge\alpha\mid\alpha\vee\alpha\]
\end{definition}

In team semantics, one usually considers a non-empty finite subset $N\subseteq \Prop$ of propositional variables and defined for valuations $v:N\to\{0,1\}$ over \( N \) and \PL-formulas \( \alpha \):
\[\llbracket \alpha\rrbracket^c:=\{v:N\to\{0,1\}\mid v\models\alpha\}.\]
We write $v\models p$ in case $v(p)=1$, and $v\not\models p$ otherwise. The valuation function $v$ is extended to the set of all \PL-formulas in the usual way.
\begin{definition}
%
For any set $\Delta\cup\{\alpha\}$ of  \PL-formulas, we write $\Delta\models^c\alpha$ if for any valuation $v$, $v\models\delta$ for all $\delta\in\Delta$ implies $v\models\alpha$. We write simply $\alpha\models^c\beta$ for $\{\alpha\}\models^c\beta$ and $\alpha\equiv^c\beta$ if both $\alpha\models^c\beta$ and $\beta\models^c\alpha$.
\end{definition}

%

%
%
%

Next we define \emph{team semantics} for \PL-formulas (cf. \cite{HannulaKVV15,YangV16}). A  team $X$ is a set of valuations for some finite set $N\subseteq \Prop$. We write $\dom(X)$ for the domain $N$ of $X$.%

\begin{definition}[Team semantics of \PL]
Let $X$ be a team. For any \PL-formula $\alpha$ with $\dom(X)\supseteq \Prop(\alpha)$, the satisfaction relation $X\models\alpha$ is defined inductively as:
\begin{itemize}
\item $X\models p$ ~~iff~~ for all $v\in X$, $v\models p$;
\item $X\models\neg p$ ~~iff~~ for all $v\in X$, $v\not\models p$;
\item $X\models \bot$ ~~iff~~ $X=\emptyset$;
\item $X\models \top$ ~~is always the case;
\item $X\models\alpha\wedge\beta$ ~~iff~~ $X\models\alpha$ and $X\models\beta$;
\item $X\models\alpha\vee\beta$ ~~iff~~ there exist $Y,Z\subseteq X$ such that $X=Y\cup Z$, $Y\models\alpha$ and $Z\models\beta$.
\end{itemize}
The set of all teams \( X \) with \( X\models\alpha \) is written as \( \modelsOf{\alpha} \).
Logical entailment and equivalence are defined as usual.%
\end{definition}

\begin{proposition}
Let $\alpha$ be a \PL-formula. Then the following properties hold:
\begin{description}
\item[Flatness:]   $X\models \alpha \iff \text{for all } v\in X,~\{v\}\models\alpha$. 
\item[Empty team property:] $\emptyset \models \alpha$.
\item[Downwards closure:] If $X\models \alpha$ and $Y\subseteq X$, then  $Y\models \alpha$.
\item[Union closure:] If  $X\models \alpha$ and $Y\models\alpha$, then $X\cup Y\models \alpha$.
\end{description}
\end{proposition}

For any \PL-formula $\alpha$, it  further holds that 
\[\{v\}\models\alpha\iff v\models\alpha,\]
and hence for classical formulas, $\Delta\models^c\alpha \iff \Delta\models\alpha$.

\smallskip\noindent\textbf{Propositional\,Dependence\,and\,Inclusion\,Logic.}
A {\em (propositional) dependence atom} is a string $\dep{a_1\dots a_k,b}$, and a {\em (propositional) inclusion atom} is a string $a_1\dots a_k\subseteq b_1\dots b_k$, in which $a_1,\dots,a_k,b,b_1,\dots,b_k$ are propositional variables from \Prop.
The team semantics of these two types of atoms is defined  as follows:
\begin{itemize}
\item  $X \models \dep{\vec{a}, b}$ ~~iff~~ for all $v, v' \in X$, $v(\vec{a})=v'(\vec{a})$ implies $v(b)=v'(b)$.
\item $X\models\vec{a}\subseteq\vec{b}$ ~~iff~~ for all $v\in X$, there exists $v'\in X$ such that $v(\vec{a})=v'(\vec{b})$.
\end{itemize}

We define {\em propositional dependence logic} (denoted as \pdl) as the extension of \PL-formulas  with dependence atoms. %
Similarly, {\em propositional inclusion logic} (denoted as \pincl) is the extension of \PL  by  inclusion atoms. In this paper,  we use \emph{propositional team logic} to refer to any of the logics \PL, {\pdl} and {\pincl}. 
\begin{proposition}\label{prop:pdl_pincl_properties}
Formulas of {\pdl} and {\pincl} have the 
empty team property. Moreover, \pdl-formulas have the downwards closure property, while \pincl-formulas have the union closure property.
\end{proposition}

A dependence atom with the empty sequence in the first component will be abbreviated as $\dep{p}$ and called {\em constancy atoms}. The team semantics of constancy atoms is reduced to

\begin{itemize}
\item  $X \models \dep{p}$ ~~iff~~ for all $v, v' \in X$, $v(p)=v'(p)$.
\end{itemize}

\begin{example}\label{example_dep_atm_propositional}
Consider the team $X$ over $\{p,q,r\}$ defined by:
\begin{center}
\begin{tabular}{c|ccc}
&$p$&$q$&$r$\\\hline
$v_1$&$1$&$0$&$0$\\
$v_2$&$0$&$1$&$0$\\
$v_3$&$0$&$1$&$0$
\end{tabular}
\end{center}
We have $X\models\dep{p,q}$ and $X\models\dep{r}$. Moreover,  $X\models\dep{p}\vee \dep{p}$ but  $X\not \models\dep{p}$. It is worth noting that \pincl-formulas $\phi$ satisfy $$\phi \equiv \phi\vee \phi$$
because of the union closure property.
\end{example}
We can define the flattening $\phi^f$ of a  \pdl-formula by replacing all dependence atoms by $\top$. It is easy to check that $\phi \models \phi^f$
and that 
\begin{equation}\label{flattening}
\{s\}\models \phi \Leftrightarrow s\models \phi^f
\end{equation}
for all assignments $s$ using the fact that dependence atoms are always satisfied by singletons.
 
\section{Preferential Logics}
\smallskip\noindent
\textbf{Preferential Models and Entailment.}
In preferential logic, an entailment  \( \varphi \nmableit \psi \) holds, when minimal models of \( \varphi \) are models of \( \psi \). This is formalized via preferential models, which we introduce in the following.

	For a  strict partial order \( {\pmP} \subseteq \mathcal{S} \times \mathcal{S} \) on a set \( \mathcal{S} \) and a subset \( S \subseteq \mathcal{S} \), an element \( s \in S \) is called \emph{minimal in \( S \) with respect to \( \pmP \)} if for each \( s' \in S \) holds  \( s' \not \pmP s \). 
	Then, \( \minOf{S}{{\pmP}} \)  is the set of all \( s \in S \)  that are minimal in \( S \) with respect to \( \pmP \).
\begin{definition}
    Let \( \mathscr{L} \) be a logic and \( \Omega \) be the set of interpretations for \( \mathscr{L} \) .
    A \emph{preferential model} for \(  \mathscr{L}  \) is a triple \( \pmW=\tuple{\pmS,\pmL,\pmP} \) where \( \pmS \) is a set, \( \ell: \pmS \to \Omega \), \( \pmP \) is a strict partial order on \( \pmS \), and the following condition is satisfied:
    \begin{itemize}
        \item[]\emph{[Smoothness]} \( \statesOf{\varphi}=\{ s \in \pmS \mid \ell(s) \models \varphi \} \) is smooth with respect to \( \pmP \) for every formula \( \varphi \in \mathscr{L} \), i.e, for each \( s \in \statesOf{\varphi} \) holds
        \begin{itemize}
        	\item \( s \) is minimal in \( \statesOf{\varphi} \) with respect to \( \pmP \) or
        	\item there exists an \( s' \in \statesOf{\varphi} \) that is minimal in \( \statesOf{\varphi} \) with respect to \( \pmP \) with \( s' \pmP s \).
        \end{itemize} 
    \end{itemize}
\end{definition}
Smoothness guarantees the existence of minimal elements.
\begin{definition}
    The entailment relation \( {\nmableitW} \subseteq \mathscr{L} \times \mathscr{L} \) for a preferential model \( \pmW \) over a logic \( \mathscr{L} \) is given by
    \begin{equation*}
    	\varphi \nmableitW \psi \text{ if for all }  s\in\minOf{\statesOf{\varphi}}{\pmP} \text{ holds } \ell(s)\models \psi
    \end{equation*}
    An entailment relation \( {\nmableit} \subseteq \mathscr{L} \times \mathscr{L} \) is called \emph{preferential} if there is a preferential model \( \pmW \) for \( \mathscr{L} \) such that  \( {\nmableit} =   {\nmableitW}  \). 
\end{definition}%
Because there are many preferential models for a logic $\mathscr{L}$, we may have for one logic $\mathscr{L}$ with multiple preferential logics that are based on $\mathscr{L}$.

\smallskip\noindent
\textbf{Axiomatic Characterization by System P.}
We make use of the following rules for non-monotonic entailment \( \nmableit \):
\begin{center}
	\begin{minipage}[b]{0.5\linewidth}
		\begin{align}
			\frac{}{\alpha\nmableit\alpha} \tag{Ref}\label{pstl:Ref}\\[0.25em]
			\frac{\alpha\equiv\beta\hspace{0.5cm}\alpha\nmableit\gamma}{\beta\nmableit\gamma} \tag{LLE}\label{pstl:LLE}\\[0.25em]
			\frac{\alpha\land\beta\nmableit\gamma\hspace{0.5cm}\alpha\nmableit\beta}{\alpha\nmableit\gamma} \tag{Cut}\label{pstl:Cut}
		\end{align}
	\end{minipage}\begin{minipage}[b]{0.45\linewidth}
		\begin{align}
			\frac{\alpha\models\beta\hspace{0.5cm}\gamma\nmableit\alpha}{\gamma\nmableit\beta} \tag{RW}\label{pstl:RW}\\[0.25em]
			\frac{\alpha\nmableit\beta\hspace{0.5cm}\alpha\nmableit\gamma}{\alpha\land\beta\nmableit\gamma} \tag{CM}\label{pstl:CM}\\[0.25em]
			\frac{\alpha\nmableit\gamma\hspace{0.5cm}\beta\nmableit\gamma}{\alpha\lor\beta\nmableit\gamma} \tag{Or}\label{pstl:Or}
		\end{align}
	\end{minipage}
\end{center}
Note that \( \models \) is the entailment relation of the underlying monotonic logic \( \mathscr{L} \).
The rules \eqref{pstl:Ref}, \eqref{pstl:RW}, \eqref{pstl:LLE},  \eqref{pstl:CM} and  \eqref{pstl:Cut} forming \emph{System C}.
The rule \eqref{pstl:CM} goes back to the foundational paper on non-monotonic reasoning system by Gabbay (\citeyear{KS_Gabbay1984}) and is a basic wakening of monotonicity. 
\emph{System P}  consists of all rules of System C and the rule \eqref{pstl:Or}.
The rule of \eqref{pstl:Or} is motivated by reasoning by case \cite{KS_Pearl1989}.
KLM
showed  a direct correspondence between preferential entailment relations and entailment relations that satisfy System~P.
\begin{proposition}[{\citeauthor{KS_KrausLehmannMagidor1990} \citeyear{KS_KrausLehmannMagidor1990}}]\label{prop:KLM_preferential_representation}
	Let \( \mathscr{L} \) be a compact Tarskian logic with all Boolean connectives.
    A entailment relation \( {\nmableit} \subseteq \mathscr{L}\times \mathscr{L} \) satisfies System P if and only if \( {\nmableit} \) is preferential.
\end{proposition} 
\section{Preferential Team Logics}
For propositional team-based logics, we restrict ourselves to preferential models that we call standard. %
\begin{definition}
    A preferential model \( \pmW = \tuple{\pmS,\pmL,\pmP} \) is called standard if 
   \begin{enumerate}[(S1)]
       \item There is no state \( s \in \pmS \) such that \( \pmL(s)=\emptyset \)
       \item For all non-empty teams \( \Int \) there is some state \( s \in \pmS \) such that \( \pmL(s)=\Int \)
   \end{enumerate}
\end{definition}

The rationale for (S1) and (S2) is to make the models concise and meaningful, i.e., containing explicit, yet necessary information for specifying reasoning. 
By (S1) we are excluding the empty team \( \emptyset \) from \( \pmS \), because team logics considered here have the empty-team property. Hence, \( \emptyset \) is trivially a model of every formula and including it provides no extra information. 
Condition (S2) ensures that every "non-trivial" model is included, and thus, its preference status is explicitly given in the preferential model.

We define the family of preferential team logics 
as those that are induced by some standard preferential model.
\begin{definition}
    A entailment relation \( \nmableit \) over some propositional team logic is called (standard) preferential, if there is some standard preferential model \( \pmW \) such that \( {\nmableit} = {\nmableitW} \).
\end{definition}

The next example is the bird-penguin example, demonstrating that preferential team logics are indeed non-monotonic.
\begin{example}\label{ex:penguin}
Fix the set of propositional variables \( N=\{ b,p,f \} \subseteq \Prop \), with the following intended meanings: \( b \) stands for \emph{\enquote{it is a bird}}, \( p \) stands for \emph{\enquote{it is a penguin}}, and \( f \) stands for \emph{\enquote{it is able to fly}}.
We construct a (standard) preferential model, by using the following teams:
\begin{align*}
    X_{{b}\overline{p}{f}} & = \begin{tabular}{c|ccc}
        &$b$&$p$&$f$\\\hline
        $v_1$&$1$&$0$&$1$\\
    \end{tabular}
    &
    X_{{b}{p}\overline{f}} & = \begin{tabular}{c|ccc}
        &$b$&$p$&$f$\\\hline
        $v_2$&$1$&$1$&$0$\\
    \end{tabular}
\end{align*}
Let \( \pmWpeng = \tuple{\pmSpeng ,\pmLpeng ,\pmPpeng } \) be the preferential model such that \( \pmSpeng  = \{ s_{\Int} \mid \Int \text{ is a  non-empty team} \} \) and \( \pmLpeng (s_X) = X \);
for all singleton teams \( X \) different from \( X_{{b}\overline{p}{f}} \) and \( X_{{b}{p}\overline{f}} \) we define:
\begin{align*}
    X_{{b}\overline{p}{f}} &  \pmPpeng X_{{b}{p}\overline{f}} \pmPpeng X &  X_{{b}\overline{p}{f}} & \pmPpeng X  \ ;
\end{align*}
for all non-empty teams \( Y \) and non-empty non-singleton teams \( Z \) we define:
\begin{equation*}
    Y \pmPpeng  Z  \text{ if } Y \subsetneq Z 
\end{equation*}
Then, for \( \nmableit = \nmableitWparam{\pmWpeng}  \) we obtain the following inference:%
\begin{align*}
    b & \nmableit f \tag{\enquote{birds usually fly}} \\
    p & \nmableit \neg f  \tag{\enquote{penguins usually do not fly}}\\
    b \land p & \notnmableit f  \tag{\enquote{penguin birds usually do not fly}}
\end{align*}
This is because we have:
\begin{align*}
    \minOf{\modelsOf{b}}{\pmPpeng} & = \{ X_{{b}\overline{p}{f}}  \} \subseteq \modelsOf{f}\\
    \minOf{\modelsOf{p}}{\pmPpeng} = \minOf{\modelsOf{b \land p}}{\pmPpeng}& = \{ X_{{b}{p}\overline{f}}  \}  \subseteq \modelsOf{\neg f} 
\end{align*}
\end{example}
Note that Example~\ref{ex:penguin} is agnostic about the concrete team logic used, i.e., it 
applies to {\PL}, {\pdl}, and {\pincl}.

\section{General Axiomatic Evaluation}\label{sec:generalresults}
We will now present general results on whether System~P hold for non-preferential and preferential team logics.
\smallskip\noindent\textbf{System~P and Non-Preferential Team Logics.}
For the non-preferential entailment \( \models \) of propositional team logics, we obtain that System~P is not satisfied by \pdl. For \PL and \pincl, we obtain that they satisfy System~P.
\begin{proposition}\label{prop:standardTS_or}
    The following statements holds for ${\models}$:
    \begin{enumerate}[(a)]
        \item \pdl satisfies System~C, but violates System~P.
        \item \PL and \pincl satisfy System~P.
    \end{enumerate}
\end{proposition}
Note that Example~\ref{example_dep_atm_propositional} is a witness for the second part of the statement (a) of Proposition~\ref{prop:standardTS_or}, i.e., \pdl violates \eqref{pstl:Or}.

\smallskip\noindent\textbf{System~P and Preferential Team Logics.}
Generally, System~C is satisfied by preferential team logics.
\begin{proposition}\label{prop:systemCteams}
    Let \( \pmW=\tuple{\pmS,\pmL,\pmP} \)  be a preferential model for a propositional team logic.
    The preferential entailment relation \( {\nmableitW} \) satisfies System~C.
\end{proposition}
The following Example~\ref{ex:violateOR} witnesses that, in general, \eqref{pstl:Or}, and hence, System~P, is violated by preferential team logics.
\begin{example}\label{ex:violateOR}
    Assume that \( N=\{p,q\}\subseteq\Prop \) holds.
    The following valuations \( v_1,v_2 ,v_3 \) will be important: %
    \begin{align*}
        v_1(p) &=v_1(q)=v_2(q)=1 & v_2(p)&= v_3(p)=v_3(q)=0
    \end{align*}
    We consider the teams \( X_{pq}=\{ v_1\} \), \( X_{\overline{p}{q}}=\{ v_2\} \), and  \( X_{p\leftrightarrow q}=\{ v_1,v_3\} \).
    Let \( \pmWpq = \tuple{\pmSpq,\pmLpq,\pmPpq} \) be the preferential model such that
    \begin{align*}
        \pmSpq & = \{ s_{\Int} \mid \Int \text{ is a  non-empty team} \} & \pmLpq(s_X) & = X         
    \end{align*}
    holds, and such that \( \pmPpq \) is the strict partial order given by
    \begin{align*}
        X_{p\leftrightarrow q} & \pmPpq  X_{pq}             & X_{pq}              & \pmPpq X \\
        X_{p\leftrightarrow q} & \pmPpq X_{\overline{p}{q}} & X_{\overline{p}{q}} & \pmPpq X
    \end{align*}
    where \( X \) stands for every team different from \(  \pmPpq  X_{pq} \), \( X_{\overline{p}{q}} \) and  \(  X_{p\leftrightarrow q} \).
We     obtain the following preferential entailments:
    \begin{align*}
        p & \nmableitWparam{\pmWpq} q &  \neg p & \nmableitWparam{\pmWpq} q & p\lor\neg p &\notnmableitWparam{\pmWpq} q
    \end{align*}
\end{example}
\begin{proposition}\label{prop:violoateOr}
    The entailment relation \( \nmableit_{\!\pmWpq} \) for \PL, respectively \pdl and \pincl, 
    violates \eqref{pstl:Or}.
\end{proposition}
We can reestablish satisfaction of System~P, by demanding the \eqref{eq:StarProperty}-property, given in the following proposition.
\begin{proposition}\label{prop:SystemPifStar}
Let \( \pmW \) be a preferential model for some preferential team logic.
If \eqref{eq:StarProperty} is satisfied for all formulas \( A,B \), then \( \nmableitW \) satisfies System~P, whereby\footnote{Abbreviation: \( \minOf{\modelsOf{A}}{\pmP} = \{ \pmL(s) \mid s\in \minOf{S(A)}{\pmP}  \} \)}:
\begin{equation}
    \minOf{\modelsOf{A \lor B}}{\pmP} \subseteq \minOf{\modelsOf{A}}{\pmP} \cup \minOf{\modelsOf{B}}{\pmP}	\tag{\( \star \)}\label{eq:StarProperty}
\end{equation}
\end{proposition}

%
%
%
%
%
%
%
%
%
%
%
%
%
%
%
%
%
%
%
%
%
%
%
%
%
%
%
%
%
%
%
%
%
%
%
%
%
%
%
%
%

%
%
%
%
%
%
%
%
%
%
%
%
%
%
%
%
%
%
%
%
%
%
%
%
%
%
%
%
%
%
%
%
%
%
%
%
%
%
%
%
%
%
%
%
%
%
%
%
%
%
%
%
%
%
%
%
%
%
%
%
%
%
%
%
%
%
%
%
%
%
%
%
%
%
%
%
%
%
%
%
%
%
%
%
%
%
%
%
%
%
%
%
%
%
%
%
%
%
%
%
%
%
%
%
%
%
%
%
%
%
%
%
%
%
%
%
%
%
%
%
%
%
%
%
%
%
%
%
%
%
%
%
%
%
%
%
%

%
%

%
%
%
%
%
%
%
%
%
%
%
%
%
%
%
%
%
%
%
%
%
%
%
%
%
%
%
%
%
%
%
%
%
 
\section{Results for Preferential Dependence Logics}\label{sec:resultsPDL}
For preferential dependence logic, we provide additional results to those of Section~\ref{sec:generalresults}.

\smallskip\noindent\textbf{System~P and Preferential Dependence Logic.}
The main contribution is a characterization of exactly those preferential entailment relations that satisfy all rules of System~P.
Central of this result is the following property for a preferential model \( \pmW =\tuple{\pmS,\pmL,\pmP} \), where \( s,s'\in\pmS \) are states:
\begin{equation}
    \!\text{for all\,} s,\, |\pmL(\!s\!)|\,{>}\,1 \text{,\,exists } s'  \text{ with } \pmL(\!s') \,{\subsetneq}\, \pmL(\!s\!) \text{ and } s'\,{\pmP}\, s \tag{\( \triangle \)}\label{eq:TriangleProperty}
\end{equation}
The \eqref{eq:TriangleProperty}-property demands (when understanding states as teams) that for each non-singleton team \( X \) exists a proper subteam \( Y \) of \( X \) that is preferred over \( X \).
For this property, we can show the following theorem.
\begin{theorem}\label{main}
    Let \( \pmW=\tuple{\pmS,\pmL,\pmP} \)  be a preferential model for \pdl.
    The following statements are equivalent:
    \begin{enumerate}[(i)]
        \item \( \nmableitW \) satisfies System~P.
        \item \( \pmW \) satisfies the \( \triangle \)-property.
        \item The \eqref{eq:StarProperty}-property holds for all \( A,B\in\pdl \).
    \end{enumerate}
\end{theorem}
\noindent For Theorem~\ref{main}, one shows the equivalence of (ii) and (iii) and that (i) implies (ii); (iii) to (i) is given by Proposition~\ref{prop:SystemPifStar}.

\smallskip\noindent\textbf{Relation to Dependence Logic and Classical Entailment.}
Theorem~\ref{main} and the \( \triangle \)-property imply that preferential dependence logics that satisfy System P are quintessentially the same as their flatting counterpart in (preferential) propositional logic with classical (non-team) semantics.
\begin{theorem}\label{col:triangle_flattening}
    Let \( \pmW=\tuple{\pmS,\pmL,\pmP} \) be a preferential model over \pdl that satisfies System~P. Then 
    $A \nmableitW B$ iff  $A^f \nmableitWparam{\pmW'} B^f$,
    where \( \pmW'=\tuple{\pmS',\pmL',\pmP'} \) denotes the preferential model for classical propositional logic \PL, i.e., over \( \models^c \) and valuations induced by the singleton teams in $W$.
\end{theorem}
As a last result, we consider preferential models that characterize the \( \models \) entailment relation, as well as the entailment relation for classical formulas \( \models^c \).
Let \( \pmWsub = \tuple{\pmSsub,\pmLsub,\pmPsub} \) and \( \pmWsup = \tuple{\pmSsup,\pmLsup,\pmPsup} \) be the preferential models such that the following holds:
\begin{align*}
    &\pmSsub = \pmSsup = \{ s_{\Int} \mid \Int \text{ is a non-empty team} \} \\
    &\pmLsub(s_X) = \pmLsup(s_X)   = X \\ 
    & Y \pmPsub X  \text{ if } Y \subsetneq X \qquad
    Y \pmPsup X  \text{ if } X \subsetneq Y
\end{align*}
In \( \pmWsub \) and \pmWsup, for each team \( X \) there is exactly one state \( s_X \) that is labelled by \( X \).
In \( \pmPsub \), subsets of a team are preferred, whereas in \( \pmPsup \) superset teams are preferred.

The preferential model \( \pmWsup \) gives rise to the \pdl entailment relation \( \models \), and the preferential model \( \pmWsup \) gives rise to classical entailment of the flattening \( \models^c \).
\begin{proposition}\label{prop:pdl_examples}
    For all \pdl-formulas \( A , B \) we have:
    \begin{enumerate}[(1)]
    	\item \( A \nmableitWparam{\pmWsub} B \text{ if and only if }  A^f \models^c B^f \)
    	\item \( A \nmableitWparam{\pmWsup} B \text{ if and only if }  A \models B \)
    \end{enumerate}
\end{proposition}
Note that, in conformance with Theorem~\ref{main} and Proposition~\ref{prop:standardTS_or}, \( \pmWsup \) violates the \eqref{eq:TriangleProperty}-property and \eqref{eq:StarProperty}-property.

\section{Conclusion}
We considered preferential propositional team logics, which are  non-monotonic logics in the style of \citeauthor{KS_KrausLehmannMagidor1990}.
Our results are a primer for further investigations on non-monotonic team logics.
We want to highlight that Theorem~\ref{col:triangle_flattening} indicates that \eqref{pstl:Or} of System~P is too restrictive for non-monotonic team logics.
In future work, the authors plan to identify further results on preferential models, especially with respect to axiomatic systems different from System~P.
Connected with that is to study the meaning of conditionals and related complexity issues in the setting of team logics. 
\clearpage
\appendix
\bibliographystyle{kr}
\bibliography{juha,referencesKS,biblio}

\begin{thebibliography}{}

\bibitem[\protect\citeauthoryear{Barlag \bgroup et al\mbox.\egroup
  }{2023}]{BarlagHKPV23}
Barlag, T.; Hannula, M.; Kontinen, J.; Pardal, N.; and Virtema, J.
\newblock 2023.
\newblock Unified foundations of team semantics via semirings.
\newblock In {\em {KR}},  75--85.

\bibitem[\protect\citeauthoryear{Brewka, Dix, and
  Konolige}{1997}]{KS_Brewka1997}
Brewka, G.; Dix, J.; and Konolige, K.
\newblock 1997.
\newblock {\em Nonmonotonic Reasoning: An Overview}, volume~73 of {\em {CSLI}
  Lecture Notes}.
\newblock {CSLI} Publications, Stanford, {CA}.

\bibitem[\protect\citeauthoryear{Ciardelli, Iemhoff, and
  Yang}{2020}]{10.1215/00294527-2019-0033}
Ciardelli, I.; Iemhoff, R.; and Yang, F.
\newblock 2020.
\newblock {Questions and Dependency in Intuitionistic Logic}.
\newblock {\em Notre Dame Journal of Formal Logic} 61(1):75 -- 115.

\bibitem[\protect\citeauthoryear{Durand \bgroup et al\mbox.\egroup
  }{2018a}]{DurandHKMV18}
Durand, A.; Hannula, M.; Kontinen, J.; Meier, A.; and Virtema, J.
\newblock 2018a.
\newblock Approximation and dependence via multiteam semantics.
\newblock {\em Ann. Math. Artif. Intell.} 83(3-4):297--320.

\bibitem[\protect\citeauthoryear{Durand \bgroup et al\mbox.\egroup
  }{2018b}]{HKMV18}
Durand, A.; Hannula, M.; Kontinen, J.; Meier, A.; and Virtema, J.
\newblock 2018b.
\newblock Probabilistic team semantics.
\newblock In {\em Foundations of Information and Knowledge Systems - 10th
  International Symposium, FoIKS 2018, Budapest, Hungary, May 14-18, 2018,
  Proceedings},  186--206.

\bibitem[\protect\citeauthoryear{Gabbay, Hogger, and
  Robinson}{1993}]{KS_GabbayHoggerRobinson1993}
Gabbay, D.~M.; Hogger, C.~J.; and Robinson, J.~A.
\newblock 1993.
\newblock {\em Handbook of Logic in Artificial Intelligence and Logic
  Programming}.

\bibitem[\protect\citeauthoryear{Gabbay}{1984}]{KS_Gabbay1984}
Gabbay, D.~M.
\newblock 1984.
\newblock Theoretical foundations for non-monotonic reasoning in expert
  systems.
\newblock In Apt, K.~R., ed., {\em Logics and Models of Concurrent Systems -
  Conference proceedings, Colle-sur-Loup (near Nice), France, 8-19 October
  1984}, volume~13 of {\em {NATO} {ASI} Series},  439--457.
\newblock Springer.

\bibitem[\protect\citeauthoryear{Galliani}{2012}]{galliani12}
Galliani, P.
\newblock 2012.
\newblock Inclusion and exclusion dependencies in team semantics: On some
  logics of imperfect information.
\newblock {\em Annals of Pure and Applied Logic} 163(1):68 -- 84.

\bibitem[\protect\citeauthoryear{Galliani}{2015}]{Galliani15}
Galliani, P.
\newblock 2015.
\newblock The doxastic interpretation of team semantics.
\newblock In {\em Logic Without Borders}, volume~5 of {\em Ontos Mathematical
  Logic}. De Gruyter.
\newblock  167--192.

\bibitem[\protect\citeauthoryear{G{\"{a}}rdenfors and
  Makinson}{1994}]{KS_GaerdenforsMakinson1994}
G{\"{a}}rdenfors, P., and Makinson, D.
\newblock 1994.
\newblock Nonmonotonic inference based on expectations.
\newblock {\em Artif. Intell.} 65(2):197--245.

\bibitem[\protect\citeauthoryear{Gutsfeld \bgroup et al\mbox.\egroup
  }{2022}]{GutsfeldMOV22}
Gutsfeld, J.~O.; Meier, A.; Ohrem, C.; and Virtema, J.
\newblock 2022.
\newblock Temporal team semantics revisited.
\newblock In Baier, C., and Fisman, D., eds., {\em {LICS} '22: 37th Annual
  {ACM/IEEE} Symposium on Logic in Computer Science, Haifa, Israel, August 2 -
  5, 2022},  44:1--44:13.
\newblock {ACM}.

\bibitem[\protect\citeauthoryear{Hannula and Kontinen}{2016}]{HannulaK16}
Hannula, M., and Kontinen, J.
\newblock 2016.
\newblock A finite axiomatization of conditional independence and inclusion
  dependencies.
\newblock {\em Inf. Comput.} 249:121--137.

\bibitem[\protect\citeauthoryear{Hannula \bgroup et al\mbox.\egroup
  }{2018}]{HannulaKVV15}
Hannula, M.; Kontinen, J.; Virtema, J.; and Vollmer, H.
\newblock 2018.
\newblock Complexity of propositional logics in team semantic.
\newblock {\em {ACM} Trans. Comput. Log.} 19(1):2:1--2:14.

\bibitem[\protect\citeauthoryear{Hannula \bgroup et al\mbox.\egroup
  }{2020}]{abs-2003-00644}
Hannula, M.; Kontinen, J.; den Bussche, J.~V.; and Virtema, J.
\newblock 2020.
\newblock Descriptive complexity of real computation and probabilistic
  independence logic.
\newblock In Hermanns, H.; Zhang, L.; Kobayashi, N.; and Miller, D., eds., {\em
  {LICS} '20: 35th Annual {ACM/IEEE} Symposium on Logic in Computer Science,
  Saarbr{\"{u}}cken, Germany, July 8-11, 2020},  550--563.
\newblock {ACM}.

\bibitem[\protect\citeauthoryear{Hannula, Kontinen, and
  Virtema}{2020}]{HannulaKV20}
Hannula, M.; Kontinen, J.; and Virtema, J.
\newblock 2020.
\newblock Polyteam semantics.
\newblock {\em J. Log. Comput.} 30(8):1541--1566.

\bibitem[\protect\citeauthoryear{Kraus, Lehmann, and
  Magidor}{1990}]{KS_KrausLehmannMagidor1990}
Kraus, S.; Lehmann, D.; and Magidor, M.
\newblock 1990.
\newblock Nonmonotonic reasoning, preferential models and cumulative logics.
\newblock {\em Artif. Intell.} 44(1-2):167--207.

\bibitem[\protect\citeauthoryear{Makinson and
  G{\"a}rdenfors}{1991}]{KS_MakinsonGaerdenfors1991}
Makinson, D., and G{\"a}rdenfors, P.
\newblock 1991.
\newblock Relations between the logic of theory change and nonmonotonic logic.
\newblock In Fuhrmann, A., and Morreau, M., eds., {\em The Logic of Theory
  Change},  183--205.
\newblock Berlin, Heidelberg: Springer.

\bibitem[\protect\citeauthoryear{Pearl}{1989}]{KS_Pearl1989}
Pearl, J.
\newblock 1989.
\newblock {\em Probabilistic reasoning in intelligent systems - networks of
  plausible inference}.
\newblock Morgan Kaufmann series in representation and reasoning. Morgan
  Kaufmann.

\bibitem[\protect\citeauthoryear{Ragni \bgroup et al\mbox.\egroup
  }{2020}]{KS_RagniKernIsbernerBeierleSauerwald2020}
Ragni, M.; Kern{-}Isberner, G.; Beierle, C.; and Sauerwald, K.
\newblock 2020.
\newblock Cognitive logics - features, formalisms, and challenges.
\newblock In Giacomo, G.~D.; Catal{\'{a}}, A.; Dilkina, B.; Milano, M.; Barro,
  S.; Bugar{\'{i}}n, A.; and Lang, J., eds., {\em Proceedings of the 24nd
  European Conference on Artificial Intelligence ({ECAI} 2020)}, volume 325 of
  {\em Frontiers in Artificial Intelligence and Applications},  2931--2932.
\newblock {IOS} Press.

\bibitem[\protect\citeauthoryear{V\"a\"an\"anen}{2007}]{vaananen07}
V\"a\"an\"anen, J.
\newblock 2007.
\newblock {\em Dependence Logic}.
\newblock Cambridge University Press.

\bibitem[\protect\citeauthoryear{Yang and Väänänen}{2016}]{YangV16}
Yang, F., and Väänänen, J.
\newblock 2016.
\newblock Propositional logics of dependence.
\newblock {\em Ann. Pure Appl. Logic} 167(7):557--589.

\end{thebibliography}

\clearpage
\section{Supplementary Material}

\newenvironment{repeatprop}[1]{\smallskip\par\noindent\textbf{Proposition~\ref{#1}}.\itshape}{\par}
\newenvironment{repeatthm}[1]{\smallskip\par\noindent\textbf{Theorem~\ref{#1}}.\itshape}{\par}
\newenvironment{repeatcorollary}[1]{\smallskip\par\noindent\textbf{Corollary~\ref{#1}}.\itshape}{\par}

\subsection{Proofs for Section~\ref{sec:generalresults}}

\begin{repeatprop}{prop:standardTS_or}
    The following statements holds for ${\models}$:
    \begin{enumerate}[(a)]
        \item \pdl satisfies System~C, but violates System~P.
        \item \PL and \pincl satisfy System~P.
    \end{enumerate}
\end{repeatprop}
\begin{proof}We show both statements.
    \begin{itemize}
        \item[]\hspace{-1em}\emph{(a)}
    Satisfaction of System~C is a corollary of Proposition~\ref{prop:systemCteams} and (b) of Proposition~\ref{prop:pdl_examples}. 
    The violation of \eqref{pstl:Or} is witnessed by choosing  $\alpha$, $\beta$ and $\gamma$ to be the formula $\dep{p}$ in Example~\ref{example_dep_atm_propositional}.
    
    \medskip
    \item[]\hspace{-1em}\emph{(b)} We start with satisfaction of System~C.
    Note that one can reconstruct non-preferential entailment \( \models \) of \PL  by using a preferential model where all teams are incomparable. In such a preferential model \( \pmW \) one has \( \minOf{\modelsOf{\alpha}}{\pmP}=\modelsOf{\alpha} \). Hence, we have \( \alpha\nmableitW \beta \) if and only if \( \modelsOf{\alpha}\subseteq\modelsOf{\beta} \) if and only if \( \alpha\models\beta \).
    By using this, satisfaction of System~C is a corollary of Proposition~\ref{prop:systemCteams}.\\
     It remains to show that \eqref{pstl:Or} is satisfied. Let \( A \), \( B \) and \( C \) be \PL-formulas such that \( A \models C \) and \( B \models C \). If \( X \) is a model of \( A \lor B \), then there are teams \( Y,Z \) with \( X = Y \cup Z \) such that \( Y \models A \) and \( Z \models B \). Because \( Y,Z \) are models of \( C \) and because \PL has the union closure property (see Proposition~\ref{prop:pdl_pincl_properties}), we obtain that \( X \) is also a model of \( C \). Hence, \( A \lor B \models C \).
    The proof of statement (b) for \pincl is the same.\qedhere
    \end{itemize}
\end{proof}

\begin{repeatprop}{prop:systemCteams}
    Let \( \pmW=\tuple{\pmS,\pmL,\pmP} \)  be a preferential model for a propositional team logic.
    The preferential entailment relation \( {\nmableitW} \) satisfies System~C.
\end{repeatprop}
\begin{proof}
    We show that \( {\nmableitW} \) satisfies all rules of System C, i.e., \ref{pstl:Ref}, \ref{pstl:LLE}, \ref{pstl:RW}, \ref{pstl:Cut}, and \ref{pstl:CM}.
    \begin{itemize}
        \item[]\hspace{-1em}[\emph{\ref{pstl:Ref}.}] 
        Considering the definition of \( \nmableitW \) yields that \( \alpha \nmableitW \alpha \) if for all minimal \( s\in\statesOf{\alpha} \) holds \( \ell(s)\models\alpha \). By the definition of \( \statesOf{\alpha} \), we have \( s\in\statesOf{\alpha} \) if \( \ell(s)\models\alpha \). Consequently, we have \( \alpha\nmableitW \alpha \).
        
        \medskip
        \item[]\hspace{-1.0em}[\emph{\ref{pstl:LLE}.}] 
        Suppose that \( \alpha\equiv\beta \) and \( \alpha\nmableitW\gamma \) holds.
        From \( \alpha\equiv\beta \), we obtain that \( \statesOf{\alpha}=\statesOf{\beta} \) holds.
        By using this last observation and the definition of \( \nmableitW \), we obtain \( \beta\nmableitW\gamma \) from \( \alpha\nmableitW\gamma \).
        
        \medskip
        \item[]\hspace{-1.0em}[\emph{\ref{pstl:RW}.}] 
        Suppose that \( \alpha\models\beta \) and \( \gamma\nmableitW\alpha \) holds.
        Clearly, by definition of \( \alpha\models\beta \) we have \( \modelsOf{\alpha}\subseteq\modelsOf{\beta} \).
        From the definition of \( \gamma\nmableitW\alpha \), we obtain that \( \ell(s)\models\alpha \) holds for each minimal \( s\in\statesOf{\gamma} \).
        The condition \( \ell(s)\models\alpha \) in the last statement is equivalent to stating \( \ell(s)\in\modelsOf{\alpha} \). Because of \( \modelsOf{\alpha}\subseteq\modelsOf{\beta} \), we also have \( \ell(s)\in\modelsOf{\beta} \); and hence, \( \ell(s)\models\beta \)  for each minimal  \( s\in\statesOf{\gamma} \).
        This shows that \( \gamma\nmableitW\beta \) holds.

        \medskip
        \item[]\hspace{-1.0em}[\emph{\ref{pstl:Cut}.}] 		
        Suppose that \( \alpha\land\beta \nmableitW \gamma \) and \( \alpha \nmableitW \beta \) holds.
        By unfolding the definition of \( \nmableitW \), we obtain 
        \(  \minOf{\statesOf{\alpha\land\beta}}{\prec}  \subseteq \statesOf{\gamma} \) 
        from \( \alpha\land\beta \nmableitW \gamma \).
        Analogously, \( \alpha \nmableitW \beta \) unfolds to 
        \( \minOf{\statesOf{\alpha}}{\prec} \subseteq \statesOf{\beta} \).
        Moreover, employing basic set theory yields that \( \statesOf{\alpha\land\beta}=\statesOf{\alpha}\cap\statesOf{\beta}\subseteq\statesOf{\alpha} \) holds.
        From \( \statesOf{\alpha\land\beta}\subseteq\statesOf{\alpha} \) and \( \minOf{\statesOf{\alpha}}{\prec} \subseteq \statesOf{\beta} \), we obtain \( \minOf{\statesOf{\alpha}}{\prec} \subseteq \statesOf{\alpha\land\beta} \).
        Consequently, we also have that \( \minOf{\statesOf{\alpha}}{\prec}=\minOf{\statesOf{\alpha\land\beta}}{\prec} \) holds.
        Using the last observation and \( \minOf{\statesOf{\alpha\land\beta}}{\prec}  \subseteq \statesOf{\gamma} \), we obtain \( \minOf{\statesOf{\alpha}}{\prec}  \subseteq \statesOf{\gamma} \). Hence also \( \alpha \nmableitW \gamma \) holds.

    \medskip
    \item[]\hspace{-1.0em}[\emph{\ref{pstl:CM}.}] 		
    Suppose that \( \alpha \nmableitW \beta \) and \( \alpha \nmableitW \gamma \) holds.
    By unfolding the definition of \( \nmableitW \), we obtain \(  \minOf{\statesOf{\alpha}}{\prec}  \subseteq \statesOf{\beta} \)  and \(  \minOf{\statesOf{\alpha}}{\prec}  \subseteq \statesOf{\gamma} \).
    We have to show that \( \minOf{\statesOf{\alpha\land\beta}}{\prec} \subseteq S(\gamma) \) holds.
    Let \( s \) be element of \( \minOf{\statesOf{\alpha\land\beta}}{\prec} \).
    Clearly, we have that \( s \in S(\alpha)  \) holds. We show by contradiction that \( s \) is minimal in \( S(\alpha) \). 
    Assume that \( s \) is not minimal in \( S(\alpha) \).
    From the smoothness condition, we obtain that there is an \( s'\in S(\alpha) \) such that \( s' \prec s \) and \( s' \) is minimal in \( S(\alpha) \) with respect to \( \prec \). 
    Because \( s' \) is minimal and because we have \( \minOf{\statesOf{\alpha}}{\prec}  \subseteq \statesOf{\beta} \), we also have that \(  s'\in S(\beta) \) holds and hence that \( s'\in S(\alpha\land\beta) \) holds. The latter contradicts the minimality of \( s \) in \( S(\alpha\land\beta) \).
    Consequently, we have that \( s \in \minOf{\statesOf{\alpha}}{\prec} \) holds. 
    Because we have \(  \minOf{\statesOf{\alpha}}{\prec}  \subseteq \statesOf{\gamma} \), we obtain \( \alpha \land \beta \nmableitW \gamma \).\qedhere
\end{itemize}
\end{proof}

In the following we abuse notation and mean by \( \minOf{\modelsOf{A}}{\pmP} \) the  set of \( \pmP \)-minimal states in \( \pmS(A) \), as well as the set of al models \( \pmL(s) \) of \( A \) for which a \( \pmP \)-minimal states \( s \) in \( \pmS(A) \) exists.
More technically correct would be to write \( \minOf{S(A)}{\pmP} \) for the former, and writing \( \{ \pmL(s) \mid s\in \minOf{S(A)}{\pmP}  \} \) for the latter.
\begin{repeatprop}{prop:SystemPifStar}
    Let \( \pmW \) be a preferential model for some preferential team logic.
    If \eqref{eq:StarProperty} is satisfied for all formulas \( A,B \), then \( \nmableitW \) satisfies System~P, whereby:
    \begin{equation}
        \minOf{\modelsOf{A \lor B}}{\pmP} \subseteq \minOf{\modelsOf{A}}{\pmP} \cup \minOf{\modelsOf{B}}{\pmP}	\tag{\ref{eq:StarProperty}}
    \end{equation}
\end{repeatprop}
\begin{proof}    
    Suppose that \( A \nmableit \gamma \) and \( B \nmableit \gamma \) holds.
    This is the same as \( \minOf{\modelsOf{A}}{\preceq} \subseteq \modelsOf{\gamma} \) and \( \minOf{\modelsOf{B}}{\preceq} \subseteq \modelsOf{\gamma} \).    
    Because \eqref{eq:StarProperty} holds, this also means that \( {\minOf{\modelsOf{A \lor B}}{\preceq}} \subseteq {\modelsOf{\gamma}}  \) holds. 
\end{proof}

\subsection{Proof of Theorem~\ref{main}}
We start with giving proof for the following theorem
\begin{repeatthm}{main}
    Let \( \pmW=\tuple{\pmS,\pmL,\pmP} \)  be a preferential model for \pdl.
    The following statements are equivalent:
    \begin{enumerate}[(i)]
        \item \( \nmableitW \) satisfies System~P.
        \item \( \pmW \) satisfies the \( \triangle \)-property.
        \item The \eqref{eq:StarProperty}-property holds for all \( A,B\in\pdl \).
    \end{enumerate}
\end{repeatthm}
We will obtain the proof of the theorem via the following lemmata.

For the first lemma, assume that $N=\{p_1,\dots,p_n\}$, and let $X$ an $N$-team. We define the following formula:
\[\Theta_X:=\bigvee_{v\in X}(p_1^{v(1)}\wedge\dots\wedge p_n^{v(n)}).\]
This formula is of crucial importance for proving Theorem~\ref{main}. It is straightforward to check the following lemma.
\begin{lemma}
     $\Theta_X$ defines the family of subteams of $X$, i.e., we have
    \[Y\models\Theta_X\iff Y\subseteq X.\]
\end{lemma}

The next lemma guarantees that for a sufficient large enough teams \( X \) exist formulas \( A,B \) such that \( X \) is a model of the disjunction \( A \lor B \), but \( X \) is not a model of \( A \) and \( B \).
We make use of the following notions: define \( \downset{X}=\{ Y \mid Y \subseteq X \} \) and \( \downset{X_1,\ldots,X_n} = \downset{\{X_1,\ldots,X_n\}}=\bigcup_{i=1}^n \downset{X_i} \)
\begin{lemma}[\( \dagger \)]\label{lem:dagger}
    For each team \( X \) with \( |X|>1 \) exists formulas \( A \) and \( B \) such that
    \begin{align*}
        X & \models A \lor B \text{ , }\\
        X & \not\models A \text{ , and}\\
        X & \not\models B
    \end{align*}
\end{lemma}
\begin{proof}
    Since \( |X|>1 \) %
    there exists non-empty \( Y,Z\subseteq X \) such that \( X=Y\cup Z \) and \( Y\neq X \) and \( Z \neq X \).
    There are formulas \( A \) and \(B \) such that \( \modelsOf{A}=\downset{Y} \) and \( \modelsOf{B}=\downset{Z} \), namely \( A=\Theta_Y \) and \( B=\Theta_Z \).
\end{proof}

We will now show that the \eqref{eq:TriangleProperty}-property and the \eqref{eq:StarProperty}-property  describe the same preferential models.
\begin{lemma}\label{triangle=star}
    Let  \( \pmW=\tuple{\pmS,\pmL,\pmP} \) be a preferential model over \pdl. %
    The preferential entailment relation \( {\nmableitW} \) over \pdl satisfies  \eqref{eq:TriangleProperty} if and only if \eqref{eq:StarProperty} is satisfied.
\end{lemma}
\begin{proof}
    Assume \eqref{eq:TriangleProperty} holds. Then it is easy to see that the minimal elements of the order $\prec$ are states that are mapped, via \( \pmL \), to singleton teams. Furthermore, by the downward closure property, for any $A\vee B$ the minimal teams satisfying the formula are all singletons. Since for singleton teams the interpretation of $\vee$ is equivalent with that of the Boolean disjunction the property  (\( \star \)) follows.
    
    For the converse, assume that  (\( \star \)) holds and let \( X \) be a team with \( |X| > 1 \). We will show that  then there is some team \( Y  \) with
    \begin{align*}
        Y & \subsetneq X  \text{ , }\\ 
        Y & \neq \emptyset  \text{ , and}\\ 
        Y & \prec X
    \end{align*}
    Because \( X \) contains at least two valuations, there exist \( Y,Z\subseteq X \) such that \( X=Y\cup Z \) and \( Y\neq X \) and \( Z \neq X \).
    By (the proof of) Lemma~\ref{lem:dagger} there are formulas \( A=\Theta_Y  \) and \( B=\Theta_Z  \) such that \( X \models A \lor B \), yet \( X \not\models A \) and \( X \not\models B \).
    Using this and \eqref{eq:StarProperty}, we obtain that \( X \notin \minOf{A \lor B}{\prec} \) holds.
    However, by smoothness of \( \prec \), the set $Mod(A \lor B)=\mathcal{P}(X)$ contains a team \( X' \) such that \( X' \prec X \). Now $X'$ is a witness for the \eqref{eq:TriangleProperty}-Property.

\end{proof}

%
%
%
%
%
%
%
%
%
%
%

%

Now we are ready to give the proof of Theorem \ref{main}.
\begin{proof}[Proof of Theorem~\ref{main}]
    By Lemma~\ref{triangle=star}, it suffices to show \eqref{eq:StarProperty}\( \Rightarrow \)\eqref{pstl:Or} and {\eqref{pstl:Or}\( \Rightarrow \)\eqref{eq:TriangleProperty}.}    
     We show each direction independently:
    \begin{itemize}
        \item[]\hspace{-1em}\emph{\eqref{eq:StarProperty}\( \Rightarrow \)\eqref{pstl:Or}.}
        This is given by Proposition~\ref{prop:SystemPifStar}.        
        
        \item[]\hspace{-1em}\emph{\eqref{pstl:Or}\( \Rightarrow \)\eqref{eq:TriangleProperty}.}
        Assume, for a contradiction, that  \eqref{eq:TriangleProperty} fails.  Then there exists a team $X$ of size $j\ge 2$ such that for all $Y\subseteq X$, $Y \not \prec X$. Let $j=l+k$ ($l,k\ge 1 $ and $l\le k$) and define
        $$\alpha:= \Theta_X\wedge (\theta\vee \cdots \vee \theta ),$$
        where $\theta :=\bigwedge_{1\le i\le n}\dep{p_i} $ and $\alpha$ has $l$ many copies of $\theta$. It is easy to check that $\alpha$ is satisfied by subteams of $X$ of cardinality at most $l$. The formula $\beta$ is defined similarly with $k$ copies of $\theta$ in the disjuncts. Now it holds that $\beta \models \beta$,   $\alpha \models \beta$ but $X\not \models \alpha, \beta $. Using reflexivity and right weakening, it follows that 
        $ \beta\nmableitW \beta$ and  $\alpha\nmableitW \beta$. On the other hand,  
        since $X$ is now a minimal model of $\alpha \vee \beta$ that does not satisfy $\beta$ we have shown $\alpha \vee \beta\notnmableitW \beta$ and that (Or) fails for \( \nmableitW \).\qedhere
    \end{itemize}
\end{proof}

\subsection{Remaining Proofs for Section~\ref{sec:resultsPDL}}
The next theorem shows that for \pdl preferential entailment relations satisfying the (Or)-rule reduce to reasoning over \PL-formulas and single assignments.
\begin{repeatthm}{col:triangle_flattening}
    Let \( \pmW=\tuple{\pmS,\pmL,\pmP} \) be a preferential model over \pdl that satisfies System~P. Then 
    $A \nmableitW B$ iff  $A^f \nmableitWparam{\pmW'} B^f$,
    where \( \pmW'=\tuple{\pmS',\pmL',\pmP'} \) denotes the preferential model for classical propositional logic \PL, i.e., over \( \models^c \) and valuations induced by the singleton teams in $W$.
\end{repeatthm}
\begin{proof}
Note first that by the assumption for all valuations $s,s'$ it holds that $s\prec' s'$ iff  $\{s\}\prec \{s'\}$.  By theorem \ref{main}, $W$ satisfies the \eqref{eq:TriangleProperty}-property and hence the minimal elements of $\prec$ are singleton teams. Hence 
    $A\nmableit B$, iff,  %
    for all minimal $\{s\}\in\modelsOf{A}:$  $\{s\}\models B$, iff,
    for all $\prec'$-minimal $s\in \modelsOf{A^f}:$  $s\models B^f$.
    The last equivalence holds due to  \eqref{flattening}.
\end{proof}

\begin{repeatprop}{prop:pdl_examples}
    For all \pdl-formulas \( A , B \) we have:
    \begin{enumerate}[(1)]
        \item \( A \nmableitWparam{\pmWsub} B \text{ if and only if }  A^f \models^c B^f \)
        \item \( A \nmableitWparam{\pmWsup} B \text{ if and only if }  A \models B \)
    \end{enumerate}
\end{repeatprop}
\begin{proof}
    \begin{itemize}We show statements (1) and (2).
        \item[]\hspace{-1em}\emph{(1)}
        Observe at first that we have \( A \nmableitWparam{\pmWsub} B \) exactly when we also have \( \minOf{\modelsOf{A}}{\pmPsub} \subseteq \modelsOf{B} \).
        Because \pdl has the downwards closure property, we also have that stating \( \minOf{\modelsOf{A}}{\pmPsub} \subseteq \modelsOf{B} \) is equivalent to stating that for all singleton teams \( \{v\} \) holds that \( \{v\} \models A \)  implies \( \{v\} \models B \).
        The latter statement is equivalent to stating that for the flattening \( A^f \) and \( B^f \) holds that for all valuations \( v \) holds that \( v \models A^f \)  implies \( v \models B^f \) (see also Section~\ref{sec:teambasedlogics}). Hence, we have \( A \nmableitWparam{\pmWsub} B \text{ if and only if }  A^f \models^c B^f \).

        \item[]\hspace{-1em}\emph{(2)}
        We obtain \( {\models}  \subseteq  {\nmableitWparam{\pmWsup}} \) immediately by the definition of \( \nmableitWparam{\pmWsup} \).
        We consider the other direction. The statement \(  A \models B \)  is equivalent to \( \modelsOf{A} \subseteq \modelsOf{B} \). Because \( \modelsOf{A} \) is downward-closed, there are (pairwise \( \subseteq \)-incomparable) teams \( X_1,\ldots,X_n \) such that \( \modelsOf{A}=\downset{X_1,\ldots,X_n} \). Because of the last property, we have that \(  A \models B \) holds exactly when \( \{X_1,\ldots,X_n\} \subseteq \modelsOf{B} \) holds.
        By construction of \( \pmWsup \) we have \( \minOf{\modelsOf{A}}{\pmPsup}=\{X_1,\ldots,X_n\}  \) for \( A \).
        Consequently, we also have that \( A \nmableitWparam{\pmWsup} B \) holds and consequently, we also have \( {\nmableitWparam{\pmWsup}}  \subseteq   {\models} \). \qedhere
    \end{itemize}
\end{proof}

\end{document}